\newif\ifarxiv\arxivtrue 
\documentclass[runningheads]{llncs}
\usepackage[T1]{fontenc}
\usepackage{graphicx}
\usepackage{hyperref}
\usepackage{color}

\usepackage{listings}
\usepackage{amssymb,latexsym,mathpartir,ifthen}
\usepackage{thm-restate}
\usepackage{tikz}

\renewcommand{\phi}{\varphi}

\newcommand{\CN}{CN}
\newcommand{\Claude}{CC} 
\lstdefinelanguage{TOY}[ANSI]{C}{  morekeywords={assert,assume,bool,true,false,input,output,skip },
  columns = fullflexible,
  moredelim=[is][\fontfamily{cmtt}\selectfont\textcolor{vccOutColor}]{/*`}{`*/},
  morecomment=[is]{//--}{//--},
  basicstyle=\footnotesize\rmfamily,    breaklines=true,
  literate= 
  {\\quad}{{$\quad\:$}}1
  {\\TT}{{$\TT\:$}}1
  {\\FF}{{$\FF\:$}}1
  {\\implies}{{$\implies$}}1
  {\\Agr}{{$\Agr\:$}}1
  {\\land}{{$\land\:$}}1
  {\\Phi}{{$\Phi$}}1
  {\\Both}{{$\Both\:$}}1
  {\\forall}{{$\forall{}$}}1
  {\\exists}{{$\exists{}$}}1
  {\\lambda}{{$\lambda{}$}}1
}

\newcommand{\code}[1]{\mbox{\lstinline[basicstyle={\footnotesize\sffamily}]{#1}}}

\newcommand{\codef}[1]{\textsf{\footnotesize#1}}

\newcommand{\dt}[1]{\emph{\textbf{#1}}} 
\newcommand{\trans}[1]{\stackrel{#1}{\rightarrow}} \newcommand{\rconfig}[1]{\langle #1\rangle} \newcommand{\config}[2]{\langle #1,\: #2\rangle} 
\newcommand{\initState}{\sigma_0}

\newcommand{\events}{\Downarrow}

\newcommand{\hbra}{
  \hbox to \columnwidth{\vrule width0.3mm height 1.8mm depth-0.3mm
    \leaders\hrule height1.8mm depth-1.5mm\hfill
    \vrule width0.3mm height 1.8mm depth-0.3mm}}
\newcommand{\hket}{
  \hbox to \columnwidth{\vrule width0.3mm height1.5mm
    \leaders\hrule height0.3mm\hfill
    \vrule width0.3mm height1.5mm}}

\newenvironment{tdisplay}[1]{\medskip  
    \noindent\textbf{\normalsize #1}\\[-.3ex]
    \hbra\\[-.4ex]
  }{\par\hket
  \medskip}

\newcommand{\cat}{\mathbin{\cdot}} \newcommand{\update}[3]{(#1\, |\, #2:\, #3)} 
\newcommand{\last}{\mathsf{last}} 
\newcommand{\vis}{\mathsf{vis}}  \newcommand{\inputs}{\mathsf{inp}}  \newcommand{\state}{\mathsf{sta}}  \newcommand{\redex}{\mathsf{redex}} \newcommand{\frun}{\mathsf{frun}} 
\newcommand{\Ve}{\mathsf{Ev}} \newcommand{\VIn}{\mathsf{In}} \newcommand{\VOut}{\mathsf{Out}} \newcommand{\RP}{\mathsf{RP}} \newcommand{\SC}{\mathsf{SC}} 
\newcommand{\know}{\mathsf{k}} \newcommand{\pknow}{\mathsf{k}_{\to}} 
\newcommand{\Ein}{\mathsf{in}}
\newcommand{\Eout}{\mathsf{out}}
\newcommand{\Etick}{\bullet}

\newcommand{\fs}{\mathit{fs}}

\newcommand{\aftqua}{\mbox{\large.}\;}

\newcommand{\some}[2]{\exists #1 \aftqua #2}

\newcommand{\len}[1]{|#1|}

\newcommand{\TT}{\mathsf{tt}}
\newcommand{\FF}{\mathsf{ff}}

\newcommand{\scolon}{\mathord{:}} 
\newcommand{\imp}{\Rightarrow}

\newcommand{\Lat}{\mathcal{L}} \newcommand{\lleq}{\sqsubseteq} \newcommand{\nlleq}{\not\sqsubseteq} \newcommand{\Z}{\mathbb{Z}} \newcommand{\N}{\mathbb{N}} \newcommand{\prefixeq}{\preceq}
\newcommand{\prefix}{\prec}
\newcommand{\suffixeq}{\succeq}

\newcommand{\Agr}{\ensuremath{\mathbb{A}}}
\newcommand{\AgrCh}[2]{\ensuremath{\mathbb{A}#1\mbox{\small@}#2}}  
\newcommand{\Both}{\ensuremath{\mathbb{B}}}
 \newcommand{\implies}{\Rightarrow}
\newcommand{\union}{\mathbin{\cup}}
\newcommand{\intersect}{\mathbin{\cap}}

\newcommand{\dom}{\mathord{dom}}
\newcommand{\rng}{\mathord{rng}}

\newcommand{\take}{\mathord{\downharpoonright}}

\begin{document}

\title{Assuming You Knew:  
Fixing an Epistemic Semantics for Flow
Policies Using Agentic AI
\ifarxiv {\normalsize \linebreak (extended version: has appendix)}\vspace*{-2.5ex}\fi
}

\titlerunning{Assuming You Knew}

\author{David A. Naumann\orcidID{0000-0002-7634-6150}
  \thanks{The author was partially supported by NSF SaTC award CNS 2426414
          and by a visiting position at Cornell University.}}

\institute{Stevens Institute of Technology, USA 
\quad \email{naumann@cs.stevens.edu}}

\maketitle  

\begin{abstract}
Many high-level security requirements are about the allowed flow of information in programs and are difficult to make precise because they involve selective downgrading.
Notions from epistemic logic have emerged as a good approach to policy semantics but a robust general framework remains elusive.  
A paper appearing in CSF 2018,
entitled ``Assuming You Know: Epistemic Semantics of Relational Annotations for Expressive Flow Policies'', attempted to provide a unifying framework---but the formalization was sketchy and a correction was announced during the conference presentation.
With aid from an agentic AI coding assistant, a corrected formalization has been machine-checked in the Rocq proof assistant.
The simplicity and generality of the framework may help  compare different policy specification styles and enforce them by leveraging existing techniques.

\end{abstract}

\section{Introduction}\label{sec:intro}

Security and privacy requirements for computing systems restrict the flow of information and involve subtle tradeoffs such as 
utility versus privacy
and strength of assurance versus cost of system development and enforcement.
The malleability of data, control, and communication in computing makes it difficult to precisely model information and its flow as needed for high assurance, but much has been learned through decades of research and practice.
An established and influential approach to formalization is  epistemic~\cite{AskarovS07,balliu2011epistemic}.  Confidentiality requirements are formulated as what an observer at a given level of privilege is permitted to know about sensitive information.  This can enable expression of downgrading policies with state/event-dependent conditions---when can who learn what information derived from secrets.  

There have been a number of technical developments 
of the epistemic approach for confidentiality and integrity,
formulated for program models including batch, reactive, concurrent, and higher-order.  
A key criterion is for it to be possible for stakeholders to audit candidate policy specs \emph{extensionally}, that is,
in terms of program behavior independent from the code of the program in question.  This is in tension with the convenience to specify downgrading by some form of program construct or annotation, which has two benefits: facilitating enforcement (e.g., by type systems) and expressing policies where program state and control point clearly represents the conditions or events under which downgrading is intended.  

In a paper we refer to as \dt{\CN}~\cite{ChudnovN18},
the authors\footnote{Apropos authorship: the authors of \CN\ collaborated on high level ideas and examples but the second author is to blame for the formal development, especially its flaws.
}
present a formalization for a reactive program model where policies are expressed using \emph{assume} statements in the code but a derived \emph{release policy} is expressed solely in terms of the input stream.  This was intended to address the aforementioned tension and explore how an observation-oriented epistemic security property can be enforced via reasoning about program annotations, specifically relational assertions and assumptions.  The main result was a theorem stating that security follows from so-called \emph{safety}, where the latter is a program-oriented property amenable to runtime monitoring and proof by induction on execution steps.  Unfortunately, as announced during the paper's presentation at CSF'18, there is a significant flaw in the published version's proof.  A corrected definition of release policy and the proof was sketched at the time~\cite{ChudnovN18}. This paper works out the details and makes further revisions culminating in a fully machine-checked proof of the main result using the Rocq proof assistant~\cite{Rocq820}.

Murray et al~\cite{murray2023assume} adapt the approach of \CN\ to concurrent C programs, featuring a policy language separate from the code and a clear separation between verification that the code satisfies the relevant annotations (policy-agnostic) and confirmation that a separately given policy spec is consistent with the code annotations. 
The theory of \cite{murray2023assume} has been machine-checked, providing high assurance for an embodiment of the approach; but the security property is constant-time, a very strong property that trivializes the complications of aligning multiple executions in reasoning about epistemic properties.  

In brief, our contributions here are as follows.
(a) An epistemic semantics that revises the one in \CN\ but still makes explicit the semantics of release policy as such.
(b) A slightly revised formulation of the safety condition from \CN, and machine-checked proof that safety implies security.
(c) Remarks on our experience using an agentic AI coding assistant to develop the Rocq proof.

Machine-checking seems essential for verification of practical systems, with all the complexities of their programming languages and platforms, and even for high level designs~\cite{BasinFMNNSZZ25}.
The story of this paper is one of many instances where even a careful theoretical development in an idealized setting has mistakes. 
Our machine-checked result confirms that the announced corrections of \CN\ are basically correct but machine-checking uncovered details that were not correct.  Although the program model and main definitions are fairly simple, some proofs are intricate as is often the case when reasoning about multi-run properties.  

\emph{Outline:}
Section~\ref{sec:examp} uses a simple example to introduce the programming and specification notations.  The example is taken from \CN\ and throughout this document we have borrowed liberally from the text of that paper.  We refer the reader to that paper for more extensive discussion of the appoach and further examples.
The formal development begins in Section~\ref{sec:programs} with programs and their annotations.
Throughout the paper we use standard math style.  The intent is to faithfully render the machine-checked development in a readable way.
Section~\ref{sec:knowl} defines release policy and security,
Section~\ref{sec:safedef} defines safety and 
Section~\ref{sec:implic} gives the main theorem.
Section~\ref{sec:related} discusses related works.
Section~\ref{sec:remarks} describes how the agentic tool  
Claude Code~\cite{claudecode2026} was guided to translate from the \LaTeX\  file of the corrected \CN\ 
to the Rocq specs and to complete the proofs.  
The original proof sketches have been updated for the revised definitions;
they are included in the extended version of this paper~\cite{naumann2026extended}.
\ifarxiv (That is, the version you are reading.)\fi

\section{Declassification by example}\label{sec:examp}

We introduce our specifications with an example, written in a small imperative reactive language inspired by~\cite{BohannonPSWZ09}.
We add relational assertions and assumptions, inspired by~\cite{Chudnov14}, for specifying information flow requirements. 

\begin{figure}[t]
  \(  \begin{array}{ll}
      \codef{input}^{P}p. &\codef{if}\: ready > 1 \: \codef{then} \quad /* user */ \\
                         & \quad \codef{assume}^L\: \Both r_1 \implies \Agr p; \\
                         & \quad \codef{if}\: r_1\: \codef{then}\: o:=p;\:\codef{else}\: o:=r_2;\:\code{fi}\\
                         & \quad \codef{output}^L \: o; \: ready:=0; \: \codef{fi} 
    \end{array}
   \)
   \quad
   \(
    \begin{array}{ll}
      \codef{input}^Rr.  &\codef{assume}^L\: \Agr r;  /* randomness */ \\
                         &\codef{if}\: ready = 0 \: \codef{then}\: r_1 := r; \\
                         &\codef{else if}\: ready = 1\: \codef{then} \: r_2:= r;\: \codef{fi} \: \codef{fi} \\
                         &ready:=ready+1; 
    \end{array}
   \)
  \caption{Randomized response sampling algorithm }
  \label{fig:randomized-response}
\end{figure}

Our example is a randomized response sampling algorithm (Figure~\ref{fig:randomized-response}). The purpose of the algorithm is to provide differential privacy and plausible deniability to the respondents to a sensitive yes/no question, preserving the statistical characteristics of the underlying data~\cite{dwork2014algorithmic}. The idea is that a respondent's answer may or may not be used depending on a random coin flip. 
We use three communication channels, Private ($P$) input that feeds binary private responses to the sampler, Random ($R$) input providing a source of secure random bits, and Public ($L$) output, that expects the anonymized responses. Programs consist of input handlers, one for each input channel. 
Variables including handler parameters are initialized to 0 at system start.
Handlers run to completion and optionally produce output. 
The first handler ($\codef{input}^P$) is the input sampler. The second handler ensures the first always has two fresh random bits.  We source randomness from a single channel and thus have to ensure there are two fresh random bits, stored in variables $r_1$ (for the coin flip) and $r_2$ for the randomized reply.
This contrived program is intended to be used in an environment where two bits of randomness are sent prior to each response bit.  Loss of additional random bits would not matter but dropping response bits would be unsatisfactory.

We can ensure security of this program by enforcing the following information flow policy. Private data is only disclosed to the public \emph{under the condition} that $r_1$ is true. We can formalize this policy by considering channel identifiers to be labels and arranging the labels in a partial order, such that for any labels $\ell_1,\ell_2$ related by the ordering $\ell_1\lleq\ell_2$ information from the input channel with a label $\ell_1$ is allowed to flow to the output channel with a label $\ell_2$. 
The attack model is a standard one. Principals are labeled.  The principal at level $\ell$ can provide inputs and observe outputs at all levels $\ell'\lleq\ell$.  The program is known to all.

In the example, the labels are ordered  such that $L \lleq P$ and $L \lleq R$. This baseline policy disallows direct disclosure of the both the randomness and the responses to the public observer. However, it also disallows the intended disclosure. We resolve this issue by relaxing the policy in two places in the program using downgrading \emph{annotations} 
in Figure~\ref{fig:randomized-response}. The annotation $\codef{assume}^L\: \Agr r$ declassifies the received bit $r$ to the public label $L$.
The notation alludes to the standard two-run semantics of dependency: the \dt{agreement} formula $\Agr r$ says a pair of states agree on the value of $r$.  So the policy says if two runs agree on $r$ then outputs on $L$ also agree.

The annotation $\codef{assume}^L\Both r_1 \implies \Agr p$ effectively relabels the private response bit $p$ with the label $L$, but only in the case the value of $r_1$ in that program execution is true. 
The formula $\Both r_1$ says $r_1$ is true (non-zero) in both runs.

Whether this is a reasonable program and policy depends on the statistical properties of the channel $R$. In general, the policy specification needs to be justified in terms of the quantitative notions of differential privacy. However, that also requires reasoning about the environment in which the program is run and the statistical properties of its inputs. 
In contrast, the specification in the code expresses what the observer is allowed to learn: the $R$ values and the $P$ values, but only under a certain condition. This can be understood in non-quantitative, epistemic terms as we formalize later. Together with the assumption of randomness of the $R$ inputs, this gives us the security guarantee.\footnote{This is not an isolated example of the issue. The security of the standard password checking example, where only the outcome of the password test but not the password itself can be declassified, depends in a large part on the length of passwords and their distribution. 
A typical policy can be expressed, in our notation, as 
$\Agr numguesses \land (\Both (numguesses<10) \implies \Agr (guess=pass))$.
If the set of possible passwords is too small, even disclosing the result of this small number of tries is unacceptable.
}

These examples express downgrading in terms of current values of program variables, although only inputs and outputs are considered to be observable. 
In \CN\ we introduce as second form of agreement (written $\Agr \ell @ n$) for agreement on the $n$th value on input channel $\ell$, to express policy directly in terms of observables.
(Auxiliary state can also be used for that~\cite{BNR08b,murray2023assume}.)
An example of its use is a simple packet sniffer that logs the source IP addresses of packets while ensuring confidentiality of payloads.
The formalization here is abridged for brevity.

\section{Programs, annotations, and event traces}\label{sec:programs}

As shown in the example we formulate policy in terms of a fixed labeling of channels and a fixed may-flow relation on labels, together with annotations to express conditional flows that would not be allowed by the baseline policy.  This section begins the technical development by formalizing these building blocks.

We assume given a partially ordered set $\Lat$ of security levels, and write $\ell \lleq \ell'$ for the ordering.
\dt{Events} $t$ are given by the grammar
\[ t ::= \Ein^\ell n \mid \Eout^\ell n \mid \Etick 
\qquad \mbox{where } \ell\in \Lat, \: n\in \Z
\]
Identifiers $t,u,v$ range over events.  
The event $\Ein^\ell n$ represents the input of value $n$ on $\ell$'s input channel;
$\Eout^\ell n$ is for $\ell$'s output channel.
The \dt{tick} event $\Etick$ lets us associate 
an event with every transition.
An event $t$ is \dt{visible} at $\ell$ if $t$ is an input or output at some level $\ell'\lleq \ell$.
Let $\VIn^\ell = \{ \Ein^{\ell'} n \mid \ell'\lleq \ell \mbox{ and } n\in \Z \}$,
let $\VOut^\ell = \{ \Eout^{\ell'} n \mid \ell'\lleq \ell \mbox{ and } n\in \Z \}$
and $\Ve^\ell = \VIn^\ell \union \VOut^\ell$.
For a sequence $ts$ of events we write $\vis^\ell(ts)$ for the sub-sequence of $\ell$-visible ones,
i.e., those in $\Ve^\ell$.

\subsection{Programs}

A program consists of a set of input handlers, one for each security level.
The program consumes a sequence of inputs, handling each one in turn.
The syntax $ \codef{input}^\ell x .\, c $ is for the handler associated with channel $\ell$, with body $c$, in which variable $x$ is initialized to the input value.
A handler may produce zero or more outputs, which may be on different channels from its input.
If a handler diverges, no further input is consumed.  
Programs are \dt{input-total} in the sense that, unless the last handler is diverging, the next input event can be consumed.

\begin{tdisplay}{Handlers \hfill $\codef{input}^\ell x . \, c$}
\begin{tabular}{l@{$\:$}r@{$\;$}ll}
$e$ & $::=$ & $n \mid x \mid e \oplus e \mid \lnot e$ 
  \quad ($n\in\mathbb{Z},x\in\mathit{Vars},
    \oplus\in\{+,*,=,\mathord{\leq},\mathord{\land},\ldots\}$) 
    \quad
    & expression \\ 
$c$ & ::= & $ \codef{skip} \mid x:=e \mid \codef{output}^\ell~e  
      \mid \codef{if}~e~\codef{then}~c~\codef{else}~c \mid \codef{while}~e~\codef{do}~c \mid c \: ; c $ & command \\ 
    & $ \mid$ & $\codef{assert}^\ell~\Phi \mid \codef{assume}^\ell~\Phi $ &  
\end{tabular}
\end{tdisplay}

An \dt{$\ell$-annotation} is a command of the form $\codef{assert}^{\ell'}~\Phi$ or
$\codef{assume}^{\ell'}~\Phi$ such that $\ell'\lleq\ell$.
Annotations have no influence on the observable behavior of the program. 
The $\ell$-annotations that are assumptions specify what the observer at level $\ell$---seeing inputs and outputs on channels $\ell'\lleq\ell$---is allowed to learn.  
Assertions are only needed for enforcement (Section~\ref{sec:safesec}).  
As shown by the introductory example, the handler for some level $\ell$ may have annotations 
for unrelated levels.

\begin{tdisplay}{Relational formulas $\Phi$ for annotations}
\begin{tabular}{l@{$\:$}r@{$\;$}ll}
$\phi$ & ::= & $e \mid  \phi \lor \phi \mid \neg\phi
              \mid \forall x.\phi$ \quad & state predicate \\[.5ex]
$\Phi$ & ::= & $\Agr e  \mid \Both \phi \mid \Both \phi \implies \Agr e $ & basic relational formula  \\[.5ex]
$\Phi$ & ::= & $ \Phi\land\Phi $ & conjunction of basic formulas
\end{tabular}
\end{tdisplay}
The form $\Both \phi \implies \Agr e $ can be used together with auxiliary variables to encode various forms of state-dependent downgrading policies~\cite{BNR08b}.

\subsection{Program semantics}

In the rest of the paper, we consider a fixed program i.e.\ set of handlers,
which is left as an implicit parameter to streamline notation.

A \dt{state} $\sigma$ maps variables to integers.
At the start of a handler execution, the designated variable is set to the input value and all other variables are unchanged (i.e., variables are global to all handlers).
\dt{Configurations} are of two kinds.
A \dt{receptive} configuration $\rconfig{\sigma}$ contains just a state $\sigma $.
An \dt{active} configuration $\config{c}{\sigma }$ represents the execution of a handler,
where $c$ represents the current control and $\sigma $ is a state.
The initial configuration is $\rconfig{\initState}$ where $\initState$ maps all variables to $0$.
We write $g \trans{t} h$ to indicate that configuration $g$ transitions to $h$ with associated event $t$.

\begin{tdisplay}{Transition semantics \hfill $g \trans{t} h$}
\vspace*{-1.5ex}
\renewcommand{\MathparLineskip}{\lineskiplimit=0.4em\relax\lineskip=0.5em plus 0.1em\relax}
\begin{mathpar}

\inferrule{ 
  \codef{input}^\ell x . c \mbox{ is the handler for $\ell$} \\
  \sigma' = \update{\sigma}{x}{n} 
 }{ \rconfig{\sigma}
    \trans{\Ein^\ell n}
    \config{c}{\sigma'}
}
\and 
\inferrule{}{
  \config{ \codef{skip} }{\sigma} \trans{\Etick} \rconfig{\sigma}
}
\and
\inferrule{\sigma(e) = n  }{ 
    \config{ \codef{output}^\ell~e }{\sigma} 
    \trans{\Eout^\ell n}
    \config{\codef{skip}}{\sigma}
}
\and
\inferrule{ 
 \sigma(e)\neq 0 
}{
 \config{ \codef{if}~e~\codef{then}~c~\codef{else}~d }{\sigma} 
 \trans{\Etick} \config{c}{\sigma}
}
\and
\inferrule{}{ 
  \config{ \codef{assume}^\ell~\Phi }{\sigma} \trans{\Etick} \config{\codef{skip}}{\sigma}
}
\and
\inferrule{}{ 
  \config{ x:=e }{\sigma} \trans{\Etick} \config{\codef{skip}}{\update{\sigma}{x}{\sigma(e)}}} 
\and
\inferrule{}{ 
  \config{ \codef{assert}^\ell~\Phi }{\sigma} \trans{\Etick} \config{\codef{skip}}{\sigma}
}

\and 
\inferrule{}{
  \config{ \codef{skip};c }{\sigma} \trans{\Etick} \config{c}{\sigma}
}
\and
\inferrule{ 
 \config{ c }{\sigma} \trans{t} \config{c'}{\sigma'}
}{
 \config{ c;d }{\sigma} \trans{t} \config{c';d}{\sigma'}
}
\end{mathpar}
\vspace*{-3ex}
\end{tdisplay}
The first rule consumes an input,
transitioning from a receptive configuration to an active one.
The updated state $\update{\sigma}{x}{n}$ maps $x$ to the input value $n$ and leaves the other variables unchanged.
The other rules all start from an active configuration,
and  produce either an output event or $\Etick$.
The second rule transitions to a receptive configuration from 
an active one in which the handler has terminated.
The remaining rules involve only active configurations.
In the rule for output we write $\sigma(e)$ for the value of expression $e$ in state $\sigma$.
Comparisons ($=,\leq$) and logical operators ($\land,\lnot$) return 1 or 0, 
and non-zero values are interpreted as true.
Omitted rules for assignment, loop, and conditional are standard,
as are the last two rules which are for sequencing.

If $c$ is different from $\codef{skip}$ 
then we can define $\redex(\config{c}{\sigma})$ to be the unique command $b$
such that $b$ is not a sequential composition and either 
$c\equiv b\,;d$ for some $d$, or $b\equiv c$
(equivalence modulo associativity of $;$). 
For example, in the penultimate rule the redex is $\codef{skip}$,
and $\redex(\config{(x:=0;y:=1);z:=2}{\sigma})$ is $x:=0$.
The last rule allows the redex in a sequence to take a step.

\begin{restatable}{lemma}{lemdeterm}\label{lem:determ} For every active configuration
there is a unique output or $\Etick$ event and unique successor configuration.
For every receptive configuration and every input event there is a unique
successor configuration for that event.  
\end{restatable}

\subsection{Program pre-runs and event traces}

Knowledge is defined in terms of the trace of events a program engages in, 
but specifications are expressed as program annotations, so we need to work with both 
computations and event traces.

A \dt{pre-run}\footnote{In the literature on epistemic logic~\cite{balliu2011epistemic} a ``run'' is a complete program execution; so we use the term pre-run as we are considering finite prefixes only.}
is a non-empty list of consecutive program configurations, 
starting with the initial configuration $\rconfig{\initState}$.
We use letters $g,h$ for configurations, and plural identifiers $gs,hs$ for finite sequences thereof---\dt{lists} for short.
Likewise we use $ts$ for lists of events, and sometimes $ins$ for lists of input events. 
We write $\len{gs}$ for the length of $gs$,
$[\,]$ for the empty list, and $\cat$ for concatenation of lists---and also for appending an element at the start or end of a list.
We also treat lists as functions from an initial segment of the naturals.  So $gs_0$ is the first configuration and $\dom(gs)$ is the set of indices $0,\ldots,\len{gs}-1$.  We write $\last(gs)$ for $gs_{\len{gs}-1}$,
and $gs\take i$ for the first $i$ elements of $gs$ (i.e., the prefix up to but not including the element $gs_i$).
We write $gs\prefixeq gs'$ to say $gs$ is a prefix of $gs'$.

For any event list $ts$, let $\inputs(ts)$ be the subsequence of input events.
Every list $ins$ of input events gives rise to a unique execution which consumes the inputs or diverges trying to. 
Before formalizing this, we define the unique list of input and output events, called a \dt{trace}, associated with a pre-run.

\begin{tdisplay}{Event trace admitted by a pre-run \hfill $gs\events ts$}
\vspace*{-2ex}
\begin{mathpar}
\inferrule{}{ [ \rconfig{\initState} ] \events [\,] }
\and 
\inferrule{gs \events ts \\
\last(gs)\trans{t} g
}{ gs \cat g \events ts \cat t }
\end{mathpar}
\vspace*{-3ex}
\end{tdisplay}

So a pre-run is just a configuration list that admits a trace.
The pre-runs of the program under consideration comprise the set 
\( \{ gs \mid \some{ts}{ gs \events ts } \} \).
It is closed under nonempty prefixes but does not contain $[\,]$. 
In subsequent definitions and results, quantifications over pre-runs implicitly 
range over this set.

As a contrived example consider the program with handlers
\[ 
\codef{input}^P x . \: \codef{output}^Q x+1 
\qquad\quad
\codef{input}^Q y . \: w:=w+x+y;\codef{output}^Q w
\]
It has a pre-run which admits the trace
$ts = [\Ein^Q 3, \Etick,\Etick,\Eout^Q 3,\Etick, \Ein^P 1, \Eout^Q 2,\Etick]$.
This trace has inputs  $\inputs(ts) = [\Ein^Q 3, \Ein^P 1]$.
The pre-run has distinct states 
$(w\scolon 0,x\scolon 0,y\scolon 0)$,
$(w\scolon 0,x\scolon 0,y\scolon 3)$,
$(w\scolon 3,x\scolon 0,y\scolon 3)$, and 
$(w\scolon 3,x\scolon 1,y\scolon 3)$.
The second and third states are in the $Q$-handler and the last in the $P$-handler.

In light of determinacy (Lemma~\ref{lem:determ}), one may expect that a pre-run $gs$ determines a unique $ts$ 
with $gs\events ts$.  
There is a minor wrinkle for receptive configurations.
In the configuration following an input transition one 
can see from the handler body which channel's handler is running---unless two channels have 
identical handlers!  Starting in Section~\ref{sec:wellspec} we impose a mild condition that 
makes handler bodies distinct
and ensures the expected result (Lemma~\ref{lem:traceDetEv}).

\subsection{Well-specified programs}\label{sec:wellspec}

Henceforth we require that the program under consideration satisfies the following condition.
A program is \dt{well-specified} provided that for every $\ell$, with handler 
$\codef{input}^\ell x . c$, we have
(a) $c$ has the form $\codef{assume}^\ell~\Agr x \,; d$, for some command $d$; and 
(b) every output command $\codef{output}^{\ell'} e $ in $d$  
is immediately preceded in sequence by $\codef{assert}^{\ell'}~\Agr e $.
Let us call these \dt{boilerplate annotations}.  
They help streamline the formalization by encoding the baseline policy designated by channel labels and the ordering on labels.\footnote{For a practical notation one would bake the effect of these annotations into the semantics and explicit annotations would then be needed only for declassification.}
Boilerplate assumptions are used in the definition of release policy (Section~\ref{sec:release}).
Boilerplate assertions are used to make security enforceable (Section~\ref{sec:safesec}).

The requirement loses no generality:
The requisite annotations can be added without altering the behavior---except for 
adding $\Etick$ events, which are not observable.  
Here is the contrived example but with boilerplate annotations added.
\[\begin{array}{l}
\codef{input}^P x . \: 
  \codef{assume}^P\: \Agr x; 
  \codef{assert}^Q\: \Agr (x+1); 
  \codef{output}^Q x+1 \\
\codef{input}^Q y . \: 
  \codef{assume}^Q\: \Agr y; 
  w:=w+x+y;
  \codef{assert}^Q\: \Agr w; 
  \codef{output}^Q w
\end{array}
\]

\begin{restatable}{lemma}{lemtraceDetEv}\label{lem:traceDetEv} If $gs \events ts$ and $gs \events ts'$ then $ts=ts'$.
\end{restatable}

In light of Lemma \ref{lem:traceDetEv} we can define the \dt{input history}, $\inputs(gs)$, 
of a pre-run $gs$.
It is simply $\inputs(ts)$ where $ts$ is the unique event list with $gs\events ts$.

Owing to the possibility of divergence, input-totality does not imply that 
every input list can be fully consumed.
However, an input list produces a unique possibly infinite execution, so we can talk about the pre-runs that arise in consuming an input list $ins$.
The gist is that we build $gs$ by executing the handlers on $ins_0$, $ins_1$, \ldots, until either
a handler diverges or all inputs have been consumed.  
In order to describe $gs$ in a way that is unique in the divergence case, without recourse to infinite sequences,
the following result truncates a divergent pre-run just before the start of the divergent handler invocation.
If some input, say $ins_k$, causes its handler to diverge, 
that is described in terms of a series of finite prefixes of the diverging sequence: each prefix ends with steps of the handler for $ins_k$, and has thus consumed $ins\take(k+1)$.

\begin{restatable}{lemma}{lemtraceFromInput} \label{lem:traceFromInput} 
For any list $ins$ of input events, there is a unique pre-run $gs$ such that
(a) $\inputs(gs) \prefixeq ins$,
(b) $\last(gs)$ is receptive, 
and (c) if $\inputs(gs) \neq ins$, i.e., $\inputs(gs)=ins\take k$ for some
$k < \len{ins}$,
then there is an infinite sequence $G$ of lists of active configurations
such that, for all $i,j$ we have: 
$\len{G_i}= i+1$, $gs\cat G_i$ is a pre-run,
$\inputs(gs\cat G_i) = ins\take(k+1)$,
and $i<j \imp G_i \prefixeq G_j$.
\end{restatable}

Consider the example program preceding Lemma~\ref{lem:traceDetEv}.
For input list $ins = [\Ein^P 1, \Ein^Q 3]$ 
there is a pre-run $gs$ that ends receptive and satisfies
$\inputs(gs)=ins$.  Its event trace is 
$[\Ein^P 1, \Etick, \Etick, \Etick, \Etick, \Eout^Q 2, \Etick,
\Ein^Q 3, \Etick, \Etick, \Etick, \Etick, \Etick, \Etick, \Eout^Q 4, \Etick]$.

\section{Knowledge semantics for declassification} \label{sec:knowl}

In this section we define what it is for a pre-run to be secure.   
Section~\ref{sec:baselineknowl} defines the notion of knowledge at a level $\ell$: What can be learned by observing channels visible to $\ell$ and reasoning about the program, which is known to the observer.
Section~\ref{sec:release} defines the notion of release policy at $\ell$, expressed by 
$\ell$-assumptions in the program, and defines security, which says that knowledge is
gained only in accord with release policy.  

\subsection{Knowledge}\label{sec:baselineknowl}

Conceptually, the $\ell$-observer decides what inputs to provide at levels $\ell'\lleq\ell$, and learns from observing outputs on channels at levels $\ell'\lleq\ell$.  
Following the terminology of Beringer~\cite{Beringer12},
the \dt{major} pre-run or trace is the actual execution and \dt{minor} ones are the possible executions 
about which the observer reasons.
For a given list $ts$ of events, the $\ell$-observer's knowledge of the inputs (on all channels)
is based on their observation of the $\ell$-visible inputs and outputs.

\begin{tdisplay}{Knowledge from a trace \hfill $\know^\ell$}
\vspace*{-1.5ex}
\[ \know^\ell(ts) = 
  \{ \inputs(us) \mid \some{gs}{ gs\events us \land \vis^\ell(ts) \prefixeq \vis^\ell(us) }  \} 
\]
\vspace*{-3ex}
\end{tdisplay}
The condition captures knowledge as follows.  Having observed $\vis^\ell(ts)$, the observer knows that the possible inputs are those that arise from some pre-run $gs$ of the program with trace $us$ that is consistent with what was seen so far.
Note that $\know^\ell(ts)$ is nonempty provided $ts$ is a feasible trace (i.e., arises from some pre-run).
One may say $\know^\ell(ts)$ captures uncertainty: the observer does not know which element of 
$\know^\ell(ts)$ is the actual inputs (i.e., $\inputs(ts)$).

\begin{restatable}{lemma}{lemknowlSame}\label{lem:knowlSame} (a) If $t$ is not $\ell$-visible 
then $\know^\ell(ts\cat t) = \know^\ell(ts)$.
(b) If $t\in\Ve^\ell$ then $\know^\ell(ts\cat t) \subseteq \know^\ell(ts)$.
\end{restatable}

In the case of a visible output event $t$, the difference $\know^\ell(ts)\setminus \know^\ell(ts\cat t)$ represents what is 
learned by observing $t$.
The baseline policy expressed by labels---noninterference---can be expressed by saying nothing is ever learned
from outputs: for all $ts$ and all output events $t$, $\know^\ell(ts)\setminus \know^\ell(ts\cat t) = \emptyset$ or 
equivalently $\know^\ell(ts\cat t) \supseteq \know^\ell(ts)$.
In the case of a visible input event, the observer does ``learn''; conceptually they cause the event.  
We return to this later. 

The notion of knowledge defined above is progress-sensitive:
$\know^\ell(ts)$ excludes input lists that drive the program to divergence before the last input is handled.
In reality, one cannot observe the absence of progress---only stronger properties such as the passage of a fixed amount of time.  Moreover, enforcement of progress-sensitive security can be costly and restrictive.
So we aim for a progress-insensitive security property.
A simple way to formulate such a property is to entirely disregard diverging pre-runs
(as in \cite{AskarovS07,LiZhang22}),
but this fails to address the security of visible events prior to divergence.
A more nuanced approach (e.g., \cite{AskarovSabelfeld09}) formulates security in a way that essentially declassifies, at each step of execution, the fact whether the program will diverge without producing or 
enabling the consumption of further visible events.  
In this approach one formulates a notion of \dt{progress knowledge} $\pknow^\ell$ defined 
like $\know^\ell$ except that in addition the observer knows there will be another $\ell$-visible event.

\begin{tdisplay}{Progress knowledge \hfill $\pknow^\ell$}
\vspace*{-2ex}
\[ \pknow^\ell(ts) = 
  \{ \inputs(us\cat t\cat vs) \mid
         t\in\Ve^\ell \land \some{hs}{hs\events us\cat t\cat vs \land \: \vis^\ell(ts) \prefixeq \vis^\ell(us) } \}
\]
\vspace*{-4ex}
\end{tdisplay}

\begin{restatable}{lemma}{lemknowPknow}\label{lem:knowPknow} (a) $\pknow^\ell(ts) \subseteq \know^\ell(ts)$, and 
(b) if $t\in\Ve^\ell$ then
$\know^\ell(ts\cat t) \subseteq \pknow^\ell(ts)$.
\end{restatable}

Progress-insensitive noninterference definitions in the literature are roughly equivalent to 
requiring that 
$\know^\ell(ts\cat t)\supseteq \pknow^\ell(ts)$ holds for all $ts$ and all visible outputs $t$.
We cannot do an exact comparision with, say, 
ID security of Bohannon et al.~\cite{BohannonPSWZ09}
because our model is not exactly the same.

\subsection{Release policy and security}\label{sec:release}

A relational formula $\Phi$ is interpreted as a relation on states.
(In \CN\ the states are paired with input lists, to interpret 
channel agreements.)

\begin{tdisplay}{Semantics of relational formulas 
\hfill $\sigma \mid \sigma' \models \Phi$}
\vspace*{-2ex}
\[\begin{array}{l@{\quad}c@{\quad}l}
\sigma | \sigma' \models \Agr e & \mbox{iff} & \sigma (e)=\sigma'(e) 
\\
\sigma | \sigma' \models \Both \phi & \mbox{iff} & 
\sigma \models \phi \mbox{ and } \sigma' \models \phi
\\

\sigma | \sigma' \models \Both \phi \!\implies\! \Agr e  
 & \mbox{iff} &  \sigma | \sigma' \models \Both \phi 
  \mbox{ implies } \sigma | \sigma' \models \Agr e   \\

\sigma | \sigma' \models \Phi_0\land\Phi_1 
 & \mbox{iff} & 
  \sigma | \sigma' \models \Phi_0 \mbox{ and } 
  \sigma | \sigma' \models \Phi_1 

\end{array}
\]
\noindent 
$g \mid h\models \Phi $ iff 
$\state(g)\mid \state(h)\models \Phi $
(where $\state$ projects the state of a configuration)
\vspace*{-3ex}
\end{tdisplay}

Next we define the release policy $\RP^\ell(gs)$ for a pre-run $gs$.  Like knowledge, it is a set of input lists.
The policy says the $\ell$-observer is allowed to know which $\ell$-assumptions were reached as well as the information they convey (the value of expressions in agreement).
For boilerplate $\ell$-assumptions this merely captures that the observer 
knows the visible inputs.  But for non-boilerplate $\ell$-assumptions it embodies their meaning as specifying intended declassifications.

First some technical notions.

\begin{tdisplay}{Alignment, semi-proper alignment for $\ell$, coverage}
For pre-runs $gs,hs$, an \dt{alignment} from $gs$ to $hs$ 
is a relation $\alpha \subseteq \dom(gs)\times \dom(hs)$ that 
\textbf{(a)} is monotone, i.e., $\forall i,j,k,l$ with $i \alpha j$ and $k\alpha l$, $i<k \imp j\leq l$ and $j<l \imp i\leq k$; and \textbf{(b)} has prefix-closed domain and range, i.e.,
$i\in \dom(\alpha)$ (resp.\ $rng(\alpha)$) and $0\leq j<i$ imply $j\in \dom(\alpha)$ (resp.\ $\rng(\alpha)$).

A \dt{semi-proper alignment for $\ell$}  from $gs$ to $hs$ is an alignment $\alpha$ such that 
for all $i,j$ with $i\alpha j$,
if either $\redex(gs_i)$ or $\redex(hs_j)$ is an $\ell$-assumption
then $\redex(gs_i) = \redex(hs_j)$.

A \dt{semi-$\ell$-aligned pre-run pair} is a triple $(gs,hs,\alpha)$ where $\alpha$ is a semi-proper alignment for $\ell$ from $gs$ to $hs$.

For $\alpha$ to \dt{cover} the major pre-run $gs$
(resp.\ the minor pre-run $hs$) 
means that $\dom(\alpha) = \dom(gs)$ 
(resp.\ $\rng(\alpha)=\dom(hs)$).
\end{tdisplay}
The condition $\redex(gs_i) = \redex(hs_j)$ is meant to express that $gs_i$ and $hs_j$ are 
both active configurations, with exactly the same assumption: not just the same level and formula but 
exactly the same occurrence in the program text, i.e., point in control flow. 
(To make that precise one can provide each annotation occurrence with a unique identifying label; we do so in the Rocq formalization.)
Figure~\ref{fig:preruns} shows an alignment which is semi-proper. (It also happens to
satisfy the definition of proper in Section~\ref{sec:safedef}).

\begin{tdisplay}{Semi-conformance \hfill $\SC^\ell$}
A \dt{semi-conformance at $\ell$} is a semi-$\ell$-aligned pre-run pair $(gs,hs,\alpha)$ where 
$\alpha$ covers both $gs$ and $hs$, and for all $i,j$, if $i\alpha j$ 
and $\redex(gs_i)$ is an $\ell$-assumption $\Phi$ then
$gs_i \mid hs_j \models\Phi$.
We write $\SC^\ell(gs,hs,\alpha)$ in this case.
\end{tdisplay}

The key definition makes explicit the semantics of declassification policy as expressed by assumptions.

\begin{tdisplay}{Release policy for a pre-run \hfill $\RP^\ell$}
\vspace*{-2ex}
\[ \RP^\ell(gs) = \{ \inputs(us) \mid
                     \some{hs,\alpha}{\SC^\ell(gs,hs,\alpha)\land hs \events us } 
                  \}
\]
\vspace*{-4ex}
\end{tdisplay}

Owing to the boilerplate assumptions and coverage conditions, $\RP^\ell(gs)$ only includes $ins$ such that
$\vis^\ell(ins) = \vis^\ell(\inputs(gs))$.  Owing to non-boilerplate assumptions it only includes $ins$ from pre-runs that agree with respect to the values of declassified expressions,
which may have been influenced by non-observable inputs (e.g., the $L$-assumption in the $P$-handler in Figure~\ref{fig:randomized-response}).

\begin{tdisplay}{Secure pre-run}
A pre-run $gs$ is \dt{secure} at level $\ell$ provided that for every prefix
$hs\cat g \prefixeq gs$ (with $hs$ nonempty) 
the step to $g$ is secure at level $\ell$.
For the step to be secure at $\ell$ means that, for all $ts,t$ if $hs\events ts$ and $hs\cat g \events ts \cat t$ then 
\begin{equation}\label{eq:secure}
\know^\ell(ts\cat t) \supseteq \pknow^\ell(ts) \intersect \RP^\ell(hs\cat g)
\end{equation}
\vspace*{-4ex}
\end{tdisplay}
Recall that $ts,t$ are uniquely determined by $hs,g$ (see Lemma~\ref{lem:traceDetEv}).
For $t\notin\Ve^\ell$, (\ref{eq:secure}) always holds
because $\pknow^\ell(ts) \subseteq \know^\ell(ts) = \know^\ell(ts\cat t)$
by Lemmas~\ref{lem:knowPknow}(a) and~\ref{lem:knowlSame}(a).
For $t\in\Ve^\ell$, always 
$\know^\ell(ts\cat t) \subseteq \pknow^\ell(ts)$ by Lemma~\ref{lem:knowPknow}(b),
so (\ref{eq:secure}) says that $\know^\ell(ts\cat t)$ is no smaller---the learning no greater---than allowed by $\RP^\ell$.

As a special case, suppose the only assumptions are boilerplate assumptions.  
If $t$ is a visible input then 
$\know^\ell(ts\cat t)$ will not be a superset of $\pknow^\ell(ts)$.
But knowledge of that input is embodied by $\RP^\ell(hs\cat g)$ due to its boilerplate assumption,
as noted earlier.  
So in this case the security condition captures termination-insensitive noninterference.

\section{Safety and Security} \label{sec:safesec}

In this section we define a notion called safety,
adapted from Chudnov et al.~\cite{Chudnov14},
which connects a major pre-run to minor pre-runs with alternate input histories.
This makes no reference to observations or release policy but instead directly 
interprets annotations in terms of aligned steps of the major and minor pre-runs.
Safety entails security and is amenable to static verification and dynamic enforcement. 

Apropos the word ``safety'', it will be defined in terms of major/minor pre-run pairs in a way akin to 2-safety, in the terminology of trace and hyperproperties~\cite{ClarksonSchneiderHyper10}, but technically our safety condition is a property of pre-runs paired with input lists.

\subsection{Safety}\label{sec:safedef}

In \cite{Chudnov14} the concern is to account for all possible initial states.
Here we aim to account for all possible input lists and their associated
pre-runs.  This account is based on a classification: Given a (major) pre-run
$gs$ and an input list $ins$, the (minor) pre-run associated with $ins$ may
diverge or violate an assumption, in which case it can be disregarded---we call
those ``fiats''.  Alternatively, it may be in conformance with policy, or violate
policy due to a unalignable annotation (called alignment failure) or due to an
assertion failure.  The idea is that $gs$ is safe if there are no
minor pre-runs resulting in assertion or alignment failure.

A \dt{proper alignment for $\ell$}  from $gs$ to $hs$ is an alignment $\alpha$ such that 
for all $i,j$ with $i\alpha j$,
if either $\redex(gs_i)$ or $\redex(hs_j)$ is an $\ell$-annotation
then $\redex(gs_i) = \redex(hs_j)$.
An \dt{$\ell$-aligned pre-run pair} is a triple $(gs,hs,\alpha)$ where $\alpha$ is a 
proper alignment from $gs$ to $hs$ for $\ell$.
Figure~\ref{fig:preruns} shows an example that is a proper alignment for $Q$.
It covers $hs$.  (The program is the same as the one preceding Lemma~\ref{lem:traceDetEv}.)

\begin{remark}\label{rem:one}
The following definitions take advantage of a fine point about the transition semantics.
It is defined so that the redex in the step following an annotation is skip.
If $\redex(gs_i)$ is an annotation then $\redex(gs_{i-1})$ and 
$\redex(gs_{i+1})$ are not annotations.
Here is one consequence.  Suppose $\alpha$ is a proper alignment for $\ell$, and $i$ is the index of the $n$th $\ell$-annotation in $gs$, and
$j$ is the index of the $n$th $\ell$-annotation redex in $hs$.
If $gs_i$ or $hs_j$ is covered by $\alpha$ then $i\alpha j$ and $\alpha$ does not relate $i$ (on left) or $j$ (on right) to anything else.
\qed
\end{remark}

\begin{remark}\label{rem:two}
The definitions and results allow the presence of non-boilerplate assertions.
But they serve no purpose and the definitions are perhaps best understood 
considering the only assertions are boilerplate ones.
\qed
\end{remark}

\begin{tdisplay}{Classification of aligned pre-run pairs}
The following notions are parameterized on a given level $\ell$ and input list $ins$. 

A \dt{conformance} is an $\ell$-aligned pre-run pair $(gs,hs,\alpha)$ where 
$\inputs(hs)\prefixeq ins$, $\alpha$ covers $hs$, 
and for all $i,j$, if $i\alpha j$ 
and $\redex(gs_i)$ is an $\ell$-annotation with formula $\Phi$ then
\begin{equation}\label{eq:conform}
gs_i \mid hs_j \models\Phi
\end{equation}
Moreover either $\alpha$ covers $gs$ or $\inputs(hs)=ins$ and $\last(hs)$ is receptive.

An \dt{assumption fiat} is $\ell$-aligned $(gs,hs,\alpha)$ where
$\inputs(hs)\prefixeq ins$,
$\alpha$ covers $hs$,
and there are 
$i,j,\Phi$ such that
$i<\len{gs}$,
$j=\len{hs}-1$, 
$i\,\alpha\,j$, 
$\redex(gs_i)$  
is $\code{assume}^{\ell'} \Phi$ with $\ell'\lleq\ell$,  
Eqn.~(\ref{eq:conform}) does not hold, and  
$(gs\take i,\, hs\take j,\, \beta)$ is a conformance for $\ell,ins$ 
that covers $gs\take i$,
where $\beta = \{(k,l) \mid k\,\alpha\,l \land k < i \land l < j \}$.

An \dt{assertion failure} is the same as an assumption fiat, except that 
$\redex(gs_i)$ is $\code{assert}^{\ell'} \Phi$ with $\ell'\lleq\ell$.  

An \dt{alignment failure} is $\ell$-aligned $(gs,hs,\alpha)$ where 
$\inputs(hs)\prefixeq ins$,
and there are $i,j$ such that $i<\len{gs}$, $j=\len{hs}-1$,
$\redex(gs_i)$ is an $\ell$-annotation,  $\redex(hs_j)$ is an $\ell$-annotation different from $\redex(gs_i)$,
and $(gs\take i,hs\take j,\alpha)$ is a conformance for $\ell,ins$ that covers $gs\take i$.  Moreover, $\redex(gs_i)$ and $\redex(hs_j)$ are not both boilerplate assumptions.

An \dt{alignment fiat} is the same as an alignment failure except  $\redex(gs_i)$ and $\redex(hs_j)$ are both boilerplate assumptions.

A \dt{divergence fiat} is $\ell$-aligned $(gs,hs,\alpha)$ where
$\inputs(hs)\prefixeq ins$ and there are $i,j,k$ with $i < k < \len{gs}$ and
$j=\len{hs}-1$ such that $(gs\take (i+1), hs, \alpha)$ is a conformance (for
$\ell,ins$), $\redex(gs_k)$ is an $\ell$-annotation, and no $\ell$-annotation
occurs as redex in $gs$ between $i$ and $k$.  Moreover $i \alpha j$ and either
$i=0=j$ or $\redex(gs_i)$ is an $\ell$-annotation or $\redex(gs_i)$ and
$\redex(hs_j)$ are $\ell$-outputs.\footnote{That is, outputs at some level
$\ell'\lleq\ell$.  A nontrivial semantic fact is that under these conditions
they will in fact be the same output owing to boilerplate assertions.}  Finally, there is an infinite sequence
$H$ of configuration lists such that for all $n\in\N$, $\len{H_n} = n$, $hs\cat
H_n$ is a pre-run with $\inputs(hs\cat H_n)\prefixeq ins$, and $H_n$
contains no $\ell$-annotation or $\ell$-output.
\end{tdisplay}

\begin{figure}[t]
\begin{centering}

\(\begin{array}{l}
\codef{input}^P x . \: 
  \codef{assume}^P\: \Agr x; 
  \codef{assert}^Q\: \Agr (x+1); 
  \codef{output}^Q x+1 \\
\codef{input}^Q y . \: 
  \codef{assume}^Q\: \Agr y; 
  w:=w+x+y;
  \codef{assert}^Q\: \Agr w; 
  \codef{output}^Q w
\end{array}
\)

\vspace*{1ex}

\begin{tikzpicture}[x=0.65cm, y=1.3cm, thick,
                    every node/.style={inner sep=1.5pt, font=\boldmath}]
  \def\gap{0.6}   \foreach \s [count=\i] in {\Box,*,\circ,\ddagger,\circ,\circ,\circ,\Box,
                             *,\circ,\circ,\circ,\ddagger,\circ,\circ,\circ,\Box}
  {
    \pgfmathsetmacro{\x}{\i + \gap*(\i>1) + \gap*(\i>8)}
    \node (g\i) at (\x,1) {$\s$};
  }
  \foreach \s [count=\i] in {\Box,*,\circ,\ddagger,\circ,\circ,\circ,\Box}
  {
    \pgfmathsetmacro{\x}{\i + \gap*(\i>1)}
    \node (h\i) at (\x,0) {$\s$};
  }

  \node[anchor=east] at (0.4,1) {\textit{gs}};
  \node[anchor=east] at (0.4,0) {\textit{hs}};

  \foreach \i [evaluate=\i as \j using int(\i-1)] in {2,...,17}
    \draw (g\j) -- (g\i);
  \foreach \i [evaluate=\i as \j using int(\i-1)] in {2,...,8}
    \draw (h\j) -- (h\i);

  \foreach \a/\b in {1/1, 2/2, 3/3, 4/4, 5/5, 6/5, 7/5, 7/6, 7/7, 8/8}
    \draw[densely dashed] (g\a) -- (h\b);

    \node[gray, font=\unboldmath] (h9) at (g9 |- h8) {$*$};
  \draw[gray, thin] (h8) -- (h9);
  \draw[gray, thin, densely dashed] (g9) -- (h9);

    \tikzset{lbl/.style={font=\unboldmath\footnotesize}}
  \path (g1) -- node[above, lbl] {$\Ein^P 1$} (g2);
  \path (h1) -- node[below, lbl] {$\Ein^P 1$} (h2);
  \path (g8) -- node[above, lbl] {$\Ein^Q 3$} (g9);
  \path (g1) -- node[left,  lbl] {$\alpha$} (h1);
\end{tikzpicture}

\end{centering}
\caption{Aligned pre-runs of a well-specified program,
with $P\lleq Q$.
Dashed lines are the alignment relation, 
$\Box$ is for receptive configuration,
$*$ for assumption redex,
$\ddagger$ for assertion redex,
$\circ$ for skip, output, or assign redex.
Pre-run $gs$ is for inputs $[\Ein^P 1, \Ein^Q 3]$.
}
\label{fig:preruns}
\end{figure}
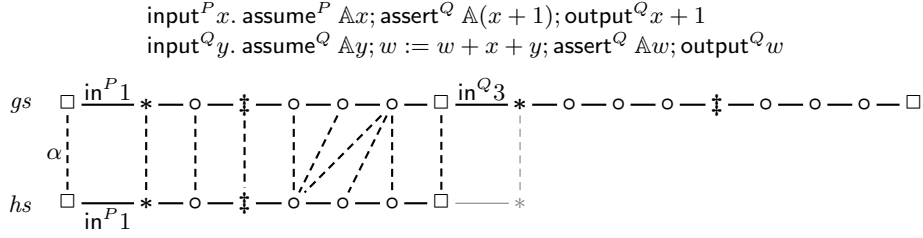

Consider the example in Figure~\ref{fig:preruns}.  If $hs$ is the pre-run for input $[\Ein^P 1]$ then it forms a conformance\footnote{Unlike semi-conformance, a conformance is not required to cover the major pre-run.}
for $Q$ and $[\Ein^P 1]$, ignoring the gray part in the diagram.  
If we include the gray part, considering that $hs$ continues by consuming $\Ein^Q 4$,
and we align its boilerplate assumption (vertical gray dashes) then we have 
an assumption failure for $[\Ein^P 1,\Ein^Q 4]$ because $3\neq 4$.  
If instead we consider that $hs$ continues by consuming $\Ein^P 0$ then it forms
an alignment fiat (omitting the gray dashes).\footnote{The alignment $\alpha$ in Figure~\ref{fig:preruns} is not unique.  For example, the non-annotation steps aligned zig-zag can also be aligned lockstep for the purposes discussed here.}

A conformance for $\ell,ins$ represents that $ins$ is possible from the point of view of an $\ell$ observer of major pre-run $gs$.
A subtlety in the definition of $(\ell,ins)$-conformance has to do with the intention to account for all steps of $gs$.
If $ins$ provides enough inputs, this can be achieved, in which case $\alpha$ covers $gs$.
If $ins$ fails to provide enough inputs for the minor pre-run $hs$ to fully align with $gs$,
nonetheless $hs$ should be as long as possible, i.e., reach a receptive configuration.
The case where $gs$ is not covered is not a satisfactory account of the security of $gs$,
but this is not a problem because the safety condition (below) quantifies over all input lists $ins$.

An assumption fiat represents that $ins$ is inconsistent with the major pre-run because it is inconsistent with 
either what is known to the $\ell$ observer (for boilerplate assumption) or meant to be released to the observer (for non-boilerplate assumption).   An assumption fiat for a boilerplate assumption represents that the observer knows 
the value of a visible input.

An alignment fiat represents that $ins$ is inconsistent with the major pre-run because the observer knows on which 
channel the input occurs: mis-alignment of boilerplate assumptions 
corresponds to inputs on different channels.

An assertion failure represents that $ins$ produces a visible output that is inconsistent with the major pre-run.  
As an example consider this insecure program, assuming $L\lleq H$.
\[ 
\codef{input}^L x . \: 
  \codef{assume}^L\: \Agr x; \codef{skip}
\qquad\qquad
\codef{input}^H y . \: 
  \codef{assume}^H\: \Agr y; 
  \codef{assert}^L\: \Agr y; 
  \codef{output}^L y
\]
For the major run on input list $[\Ein^H 0]$, the minor run on $[\Ein^H 1]$
gives an assertion failure for $L$ and $[\Ein^H 1]$.

An alignment failure represents that $ins$ drives the minor pre-run on a visibly different control path from that of the major pre-run.  A mis-aligned visible boilerplate assertion in the minor pre-run corresponds to an inconsistent output. 
A mis-aligned visible non-boilerplate assumption in the minor pre-run violates that declassifications are notionally visible,
in accord with release policy.
As an example, consider this insecure handler:
\[
\codef{input}^H y . \: 
  \codef{assume}^H\: \Agr y; 
  \codef{if}\: y\leq 0 \: \codef{then}\: \codef{assert}^L\: \Agr 3; \codef{output}^L 3
            \: \codef{else}\: \codef{assert}^L\: \Agr 5; \codef{output}^L 5 \]
(Note that $\Agr 3$ is true for any pair of states because $3=3$.)
For the major pre-run $gs$ on input list $[\Ein^H 0]$, the minor pre-run $hs$ on $[\Ein^H 1]$ 
follows a different control path and gives an alignment failure for $L$ and $[\Ein^H 1]$
because a proper alignment cannot relate two assertions at different control points.
The failure's alignment looks like this:

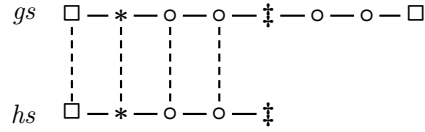
\begin{center}
\begin{tikzpicture}[x=0.65cm, y=1.3cm, thick,
                    every node/.style={inner sep=1.5pt, font=\boldmath}]
  \foreach \s [count=\i] in {\Box,*,\circ,\circ,\ddagger,\circ,\circ,\Box}
    \node (g\i) at (\i,1) {$\s$};
  \foreach \s [count=\i] in {\Box,*,\circ,\circ,\ddagger}
    \node (h\i) at (\i,0) {$\s$};

  \node[anchor=east] at (0.4,1) {\textit{gs}};
  \node[anchor=east] at (0.4,0) {\textit{hs}};

  \foreach \i [evaluate=\i as \j using int(\i-1)] in {2,...,8}
    \draw (g\j) -- (g\i);
  \foreach \i [evaluate=\i as \j using int(\i-1)] in {2,...,5}
    \draw (h\j) -- (h\i);

  \foreach \a/\b in {1/1, 2/2, 3/3, 4/4}
    \draw[densely dashed] (g\a) -- (h\b);
\end{tikzpicture}
\end{center}

Divergence fiat needs to be understood in connection with programs being well-specified and input-total.
Because $\ell$-visible inputs are accompanied by boilerplate $\ell$-assumptions, 
a divergence fiat indicates that the minor pre-run is continuing without 
producing visible output or consuming visible input.  
That could be due to a nonterminating loop with no outputs, 
or one with non-visible outputs.
By contrast, in an alignment failure the minor pre-run reaches another $\ell$-annotation,
but one that does not match the major pre-run.

\begin{tdisplay}{Safe pre-run}
A pre-run $gs$ is \dt{safe for level $\ell$ and inputs $ins$} iff
there are $hs,\alpha$ such that 
$(gs,hs,\alpha)$ is a conformance, an assumption fiat, an alignment fiat, or a divergence fiat for $\ell,ins$.
A pre-run is \dt{safe for $\ell$} if for every $ins$ it is safe for $\ell,ins$.
\end{tdisplay}

If $gs$ is safe, every alternative input list is either not a leak (owing to boilerplate asserts and conformance) 
or can be disregarded due to either (a) divergence, (b) inconsistency with visible inputs (owing to boilerplate assumptions), 
or (c) intentional release (owing to non-boilerplate assumptions).

Safety caters for reasoning by induction on execution steps, which may facilitate static verification or dynamic enforcement. 
Safety is strictly stronger than security.
For example, consider replacing 5 by 3 in the insecure handler discussed earlier.
The resulting program is secure, because the two conditional branches have the same behavior, but the alignment failure remains.  

Extending a failed or fiatted major pre-run yields the same.

\begin{lemma}\label{lem:extendFail}  If $gs\cat g$ and $gs$ are pre-runs and 
$(gs,hs,\alpha)$ is an assumption, divergence, or alignment fiat for $\ell,ins$, or an assertion or alignment failure for $\ell,ins$, then so is $(gs\cat g,hs,\alpha)$.
\end{lemma}

Prefixes of a safe pre-run are safe.

\begin{lemma}\label{lem:contractSafe} If $gs\cat g$ and $gs$ are pre-runs and $gs\cat g$ is safe for $\ell$ then so is $gs$.
\end{lemma}

The details in the definitions of failures, fiats, and conformance are motivated by the need to 
make these conditions mutually exclusive (unlike in \cite{Chudnov14})
and exhaustive, as confirmed by the following classification result.

\begin{restatable}{theorem}{lemexhaustive} \label{thm:exhaustive} 
For any $gs$, $ins$, and $\ell$  there are $hs$ and $\alpha$ such that 
$(gs,hs,\alpha)$ is either a conformance,
an assertion or alignment failure,
or an assumption, divergence or alignment fiat, for $\ell,ins$.
Furthermore, a given $gs,ins,\ell$ fits in only one category,
although there may be more than one $hs,\alpha$ that witnesses its membership in the category.
\end{restatable}
The proof goes by induction on $gs$ and constructs the requisite $hs$ and $\alpha$.
For the sake of induction hypothesis, a stronger property is proved for conformance,
namely that if $\last(hs)$ is an $\ell$-annotation then so is $\last(gs)$.  
The construction of $hs$ and $\alpha$ amounts to the design of an ideal monitor,
along the lines of the one presented in~\cite{Chudnov14} (for a simple, non-reactive imperative language).
The idea is that the monitor abstracts from a collecting semantics~\cite{RivalYiBook,AssafN16}
that tracks the classification of every $ins$ (and at every level) with respect to the current pre-run.
If any $ins$ results in an assertion or alignment failure, the
monitor flags a violation and may disallow continuing execution of $gs$.

For static verification, safety can be proved by induction on the major pre-run, considering an arbitrary $ins$ with its corresponding pre-run, along the lines of the proof of Theorem~\ref{thm:exhaustive}---the key point being to show that from a conformance the next step yields either conformance
or one of the three fiats.

\subsection{Relational safety implies epistemic security} \label{sec:implic}
 
Security says that for each observer level $\ell$, what they may learn is within what is allowed by the release policy as specified by (non-boilerplate) $\ell$-assumptions.
Our main result is that safety implies security.

\begin{restatable}{theorem}{thmsafesecure}\label{thm:safesecure} For all $\ell$, any pre-run that is safe for $\ell$ is secure for $\ell$.
\end{restatable}
A detailed proof is sketched in the appendix, including 
the assumption fiat case which was wrong in \CN.
(The appendix is in the extended version of this paper~\cite{naumann2026extended}.)
The crux is to consider the four safe possibilities for the last step in 
pre-run $hs\cat g$ in Eqn.~(\ref{eq:secure}).
We must show, for any $ins$, that 
$ins\in \pknow^\ell(ts)$ and $ins\in \RP^\ell(hs\cat g)$ imply $ins \in\know^\ell(ts\cat t)$. 
Conformance lets us establish the consequent.
Assumption fiat and alignment fiat both contradict $ins\in \RP^\ell(gs\cat g)$.
Divergence fiat contradicts $ins\in \pknow^\ell(ts)$.

\section{Related work}\label{sec:related}

The reader can consult \CN\ for relevant work as of 2018 so for brevity we just highlight a few more recent works.

Works that enable specification of state-dependent downgrading policies include~\cite{BNR08b,BrobergS10,VanhoefGDPR14,McCallZJ18,McCallB023} and typically feature epistemic security definitions.
State can be used to designate ``where'' in the code declassification occurs as well as ``what'' is declassified (using terminology from~\cite{SabelfeldSandsDeclassJ}).
Instead of an epistemic semantics, Menz et al~\cite{MenzHLG23} use logical relations for semantics of ``where'' declassification in higher-order programs.
Inspired by epistemic logics, Soloviev et al~\cite{SolovievBG24} use modal operators---both for knowledge and for agents' ability to read or write variables---to give elegant specifications of security properties including robust declassification.
Robust declassification means that whatever is declassified cannot be influenced by the adversary~\cite{ZdancewicMyers01}.

Bay and Askarov~\cite{BayA20} use an epistemic formulation to define the declassification of progress.
Li and Zhang~\cite{LiZhang22} use an epistemic style semantics for a system of dynamic policies that can express erasure which in some sense decreases knowledge. 
Cecchetti~\cite{Cecchetti25} uses infinite runs in the style of hyperproperties~\cite{ClarksonSchneiderHyper10} to define progress-sensitive and progress-insensitive variations of robust declassification and the related properties like transparent endorsement~\cite{CecchettiMA17}; the theory is machine-checked in Rocq.

\CN\ aimed for extensional interpretation of policy not overly tied with the program, despite the policy being expressed by assumptions in the code.  Murray et al~\cite{murray2023assume} take the same approach and have a security condition like
Eqn.~(\ref{eq:secure}) with an explicit release policy derived from assume commands.
They make the connection between assumptions in code and external observation via ghost state that tracks application-specific notional event traces, an approach first suggested in~\cite{BNR08b} and compatible with the present work.
But the technical development involves instrumented semantics: events track occurrences of assumptions as well as control branches, memory accesses, and thread scheduling, as needed for the strong constant-time security property.
Their theory is machine-checked in Isabelle/HOL.  As noted in Section~\ref{sec:intro} the strong property greatly simplifies the alignment of minor and major runs (and entails strong restrictions on the program).

McCall et al~\cite{McCallB023} use an epistemic security condition for monitoring of event-driven web programs where page elements and handlers can be added dynamically. Their system features state-dependent declassification akin to what can be expressed in our system and in~\cite{murray2023assume} although policies are not given a meaning independent of the security condition.  Like~\cite{murray2023assume}, the security condition refers to instrumented semantics.  A key contribution is to ensure robustness using lightweight taint tracking; robustness is proved as a corollary of the epistemic confidentiality condition.

Although our semantics does not generate instrumentation events, annotations do take steps and our release policy makes reference to those via alignments.  
The flawed definition of $\RP^\ell$ in \CN\ made it a predicate on observed event traces whereas our fix makes it a predicate on pre-runs---though still designating possible input lists.  
It seems difficult to formulate semantics for downgrading in a way that is fully extensional with respect to an un-instrumented semantics.

\section{Remarks on the use of AI}\label{sec:remarks}

In this section I switch to first person and report on my experience
developing the Rocq formalization using Claude Code~\cite{claudecode2026}
(henceforth \dt{\Claude}).
Two recent preprints~\cite{paraskevopoulou2026,Shengyi26} report on experiments using LLM-based coding assistants for substantial developments in Rocq, extending sizeable existing developments.  
Unlike those authors, I have not systematically analyzed the logs of my interactions.
My main goal was to machine-check the paper.   

I am an experienced user of Rocq and had previously used \Claude\ for a small inconclusive exploration in an existing Rocq development~\cite{NBN25subm}.  
For this paper I ran \Claude\ on an isolated virtual machine and gave \Claude\  access to bash shell commands (used for reading and editing files), the \LaTeX\ file with corrections~\cite{ChudnovN18}, the Rocq files it was developing, and the Rocq MCP server\footnote{\url{https://github.com/LLM4Rocq/rocq-mcp}} which enables \Claude\ to inspect intermediate proof goals during processing of a proof script.  

I provided an initial Rocq file with only import directives for some libraries including classical extensionality and excluded-middle axioms.  I gave the following initial prompt. 
\emph{``I want your help to create Rocq code to formalize the math in the Latex file paper-8June2026.tex.  We will do this step by step.  The math starts in the section labelled sec:programs.  Think hard and create a detailed, step-by-step plan to formalize what's in that section up to the Programs subsection''.}
I also instructed it not to alter definitions or statements of results without consulting me.
I proceeded section by section over about three weeks, averaging an hour a day, pausing the session and resuming from compacted chat summaries and saved ``memories''.  Twice I ended the session and started a new one when \Claude\ seemed bogged down, distracted by its prior English proof sketches while not remembering an overall plan.
Although I saved snapshots in a repository, \Claude\ only had access to the current version of the file. 

I was impressed that \Claude\ could develop nontrivial supporting infrastructure and do routine proofs on its own.  For example, my proof sketches are imprecise about running the program for a few steps until something is reached, e.g., an assumption, or the program diverges.  \Claude\ developed an eval-with-fuel function and lemmas about that.
\Claude\ pointed out some minor improvements while missing other obvious ones. For
the original version of  Lemma~\ref{lem:traceFromInput} it proposed a plausible but wrong formalization, which did bring my attention to a strengthing already suggested in the original expository text.

\Claude\ intermittently presented me with design decisions because it was stuck or convinced a claim was false.  I reviewed progress and occasionally corrected clear mistakes, e.g.,
its initial definitions for the classification of aligned pairs used semi-proper $\ell$-alignment where it should be proper.  While trying to prove a result akin to Lemma~\ref{lem:traceFromInput}, 
to prove Theorem~\ref{thm:exhaustive}, 
it became clear that the definition of divergence failure in \CN\ was flawed: it needed to allow additional inputs to be consumed before reaching the input whose handler diverges.  This is how it was actually used in my proof sketches, but the paper's formal definition had a constraint that every config is active.  Oddly, \Claude\ failed to spot that the best fix is simply drop that constraint.

Mutual exclusivity of the categories was vague in the original statement of Theorem~\ref{thm:exhaustive}.  \Claude\ had trouble proving exclusivity of divergence fiat from alignment fault and presented me with an unconvincing counterexample.  So I sketched a proof which helped sharpen the needed ingredients, e.g., the $n$-the visible annotation in pre-run.

For Lemma~\ref{lem:knowlSame}(b) it found a counterexample.
A proof had not been sketched in \CN\ and it's not used for main results, just a conceptual check that knowledge is defined sensibly.  Similar remarks apply to Lemma~\ref{lem:knowPknow}(b).
\Claude's proposed fixes included to require the program to be safe (noninterferent, it said, showing that it knows the literature), which would be  ridiculous.  
I gave it a proof sketch in which it found a flaw.  This interaction led me 
to devise the current definitions of $\know^\ell$ and $\pknow^\ell$, using prefix 
where \CN\ used equality. 
(Prefix is also explicit in the security condition of~\cite{McCallB023}.)

When \Claude\ was stuck proving one of the cases for theorem~\ref{thm:safesecure} 
it proposed to strengthen the theorem to say: if the program is safe for all pre-runs then it is secure for any pre-run.  While not ridiculous, that would be a less informative and less useful result.  I solved the difficulty: after noting a proof step where \Claude\ made excessively greedy use of a lemma, I sketched an alternative argument that works.

\Claude\ was prone to overly specific and intricate lemmas, where humans would step back and think about abstraction and elegance.  For auxiliary definitions it was prone to using decidability hypotheses which is gratuitous because my initial Rocq imports included classical axioms.
Here is an example of missing an obvious generalization.  
With the revised definition of knowledge, a lemma in \CN\ about visible input no longer holds.  It was never needed, just conceptual, and tied with the fact that  the definition of security excluded visible inputs.
(With informal rationale that the observer of course learns what input they choose.) 
When revising the proof sketches after \Claude\ had completely proved the main theorem,
I realized there's no longer a need to exclude visible input in the definition of security.  In a fresh \Claude\ session I asked it to drop the hypothesis $t\notin\VIn^\ell$ from definition of \verb+step_secure+ and plan a revised proof.  It appeared to go off in the weeds so I had to point out how the hypothesis was also unnecessary in subsidiary lemmas it had introduced for the proof of Theorem~\ref{thm:safesecure}.
Similarly, Lemmas~\ref{lem:knowlSame}(b) and~\ref{lem:knowPknow}(b) had also been unnecessarily restricted to outputs.

\Claude\ was keen to summarize progress and achieve progress in terms of number of Qed results.
It was prone to cheating by strengthening assumptions in its lemmas (an oft-reported issue~\cite{paraskevopoulou2026}) and by factoring out the hard part of a proof into a separate lemma.
The sycophantic and over-confident tone of chatbots has been widely remarked and applies to \Claude.  Occasionally it drifted from human-like reasonable commentary to skating on thin ice.  In a summary of progress it extruded this text: 
``So the substantive mechanization is complete: the security theorem \verb+safe_implies_secure+, the classification  (\verb+classification_exists+), the whole annotation-rank cluster, the conformance-extension lemmas, and all the exclusivity lemmas are proven---resting only on the determinacy-axiom layer the development always intended to keep abstract.''
To which I responded: \emph{``I'm not sure where you got the idea that `the development always intended to keep abstract'.  Developments don't have intentions.  As a user I do have intentions, and in particular my intention is to fully machine-check everything.''}

The proof scripts make pedestrian but intricate use of basic tactics to construct proofs. Though scripts are notoriously unreadable, humans have tactic usage styles to enhance readability for maintenance and adaptation.  But with the advent of agentic AI coding tools such considerations may shift (as noted by Paraskevopoulou~\cite{paraskevopoulou2026}).
Wang~\cite{Shengyi26} emphasizes the importance to use the coding tool to assist with cleanup, which I have not yet done.
Human curation is surely needed for comments in the code.  \Claude\ seems prone to documenting lemmas with wordy comments that are obscure due to references to some context where the lemma is used.
That is also a sign that the lemmas do not factorize the proofs well.

Despite extruding many words that gave the impression of understanding, \Claude\ often got off track in ways like the ones mentioned above.  In addition, it failed to
suggest the idea of alignment fiat, which only dawned on me when responding to requests by reviewers to improve the explanation of alignment failure.
The fiat case had previously been included in alignment failure, which validates 
Theorem~\ref{thm:safesecure} but makes for a severely restrictive notion of safety
for programs that have multiple channels visible at some level.

The Rocq development is available at~\cite{naumann_2026_22970019}.
It mechanizes Theorem~\ref{thm:safesecure} and all results stated in the paper.  
The development currently comprises 10577 non-blank lines of code and 1220 lines of comments.  
I suspect that, even without aggressive use of tactic automation, the development can be done in a third of that.

Use of \Claude\ probably saved me time, especially if I amortize the time spent installing and getting familiar with the tools and improving my prompting.  On the other hand, with a human collaborator, even an inexperienced student, our discussions might have sparked interesting new ideas.

\section{Conclusion}

We presented a machine-checked knowledge based theory that accounts for declassification of secrets in accord with policy expressed by assume statements in code akin to \verb+declassif+ statements in prior works (e.g.,~\cite{LiuAGM17,AcayRGMS21}).
The theory gives an explicit meaning to policy that is disentangled from the definition of security.
The main result says that a condition called safety, 
amenable to monitoring and inductive proof, implies security.
The development is a corrected version of the paper 
\CN~\cite{ChudnovN18} which also sketched the idea
of static verification using relational Hoare logic.
However, there is a considerable difference between 
our safety condition and the pre-post relational properties for which such logics have been developed~\cite{NaumannISOLA20}. One difference is that it explicitly asserts existence of a minor run whereas most relational Hoare logics are for an $\forall\forall$ property (though some handle  $\forall\exists$, see~\cite{NagasamudramBN25lmcs,Beutner24,DardinierM24}).
Another difference is that it does not connect a pair of executions but an execution and an input trace.  Developing compositional proof rules for safety is left as an open problem.

This paper is dedicated to David Basin on the occasion of his retirement.
He recognized early on the importance of machine-checked formalization for security. His pioneering works (e.g.,~\cite{SchmidtEtalCSF12}) have helped bring that vision to practical fruition~\cite{BasinFMNNSZZ25}.

\begin{credits}
\subsubsection{\ackname}
I am grateful to Andrey Chudnov for collaborating on the original paper~\cite{ChudnovN18} and to the anonymous referees for helpful suggestions and questions.

\subsubsection{\discintname} The author has no competing interests to declare.
\end{credits}

\bibliographystyle{splncs04}
\bibliography{biblioCSF}

\ifarxiv
\newpage
\appendix

\section{Proofs}\label{sec:proofs}

The following proof sketches are very similar to the original \LaTeX\  document that was given to \Claude.  But they have been revised to be consistent with the machine-checked definitions and results.

\lemdeterm*  \begin{proof}
Direct from definitions.  
\end{proof}

\lemtraceDetEv*
\begin{proof}
By induction on $gs$ and cases on transition rules.
For steps other than input, the argument is direct from Lemma~\ref{lem:determ}.
For input, consider a step from $\rconfig{\sigma}$ to $\config{c}{\sigma'}$.  
By semantics, $c$ is the handler body. 
Because the program is well-specified, $c$ begins with $\codef{assume}^\ell~\Agr x$, from which we have that the channel is $\ell$ and the input value was assigned to variable $x$.
So the event is $\Ein^\ell n$ where $n$ is $\sigma'(x)$.
\end{proof}

For subsequent proofs it is helpful to generalize pre-runs to execution sequences that start from arbitrary initial configuration.
Then we define $\frun(g,ins,n)$, by recursion on $n\in\N$, to be the consecutive sequence of configurations starting from $g$ (but excluding $g$) and consuming inputs $ins$ until $n$ steps have been taken or all of $ins$ have been handled and a receptive configuration is reached.

\lemtraceFromInput* \begin{proof}
If there is some $n$ with $\len{\frun(\rconfig{\initState},ins,n)}<n$ then we get (a) and (b) straightforwardly,
and have $ins$ consumed which falsifies the antecedent of (c).
If $\len{\frun(\rconfig{\initState},ins,n)}=n$ for all $n$ then we get (c) as follows.
There must be some $m,k$ such that $\frun(\rconfig{\initState},ins,m)$ has consumed $k$ inputs and reached
the receptive configuration that steps to the handler for $ins_k$ which never terminates.  
So $gs$ is $\rconfig{\initState}\cat\frun(\rconfig{\initState},ins,m)$ and the $G_i$ are defined as the suffixes of 
$\frun(\rconfig{\initState},ins,m+i)$.
Uniqueness can be shown using determinacy.
\end{proof}

Lemma~\ref{lem:traceFromInput} is formulated for conceptual clarity.
To prove prove Theorems~\ref{thm:exhaustive} and~\ref{thm:safesecure}
we actually rely on similar but slightly more intricate properties along these lines:
given a pre-run one can extend it in accord with an input list to reach the next visible annotation or output, unless execution diverges.

\lemknowlSame*
\begin{proof}
(a) is direct from the definition.
For (b), suppose $ins\in\know^\ell(ts\cat t)$, 
so there are $gs,us$ with $ins = \inputs(us)$ and   
$gs\events us$ and $\vis^\ell(ts \cat t) \prefixeq \vis^\ell(us)$.
We have $\vis^\ell(ts) \prefixeq \vis^\ell(ts\cat t)$ 
by definition of $\vis^\ell$. 
So $gs,us$ witness $ins\in\know^\ell(ts)$.
\end{proof}

\lemknowPknow*
\begin{proof}
(a) is direct from the definitions.
For (b), suppose $ins \in \know^\ell(ts\cat t)$.
Thus there are $gs,us$ with
$gs\events us$, $\vis^\ell(ts\cat t) \prefixeq \vis^\ell(us)$, and $ins = \inputs(us)$.
Now $\vis^\ell(ts\cat t) = \vis^\ell(ts)\cat t$ using  $t\in\Ve^\ell$.
So there are $vs,vs'$ with $us= vs\cat t\cat vs'$ and 
$\vis^\ell(us) = \vis^\ell(vs) \cat t \cat \vis^\ell(vs')$.  
So we get $ins\in \pknow^\ell(ts)$ by instantiating the definition of $\pknow^\ell(ts)$
with $us,t,vs,hs:=vs,t,vs',gs$.
\end{proof}

\lemexhaustive*
\begin{proof}
For mutual exclusivity see the Rocq proof.  Here we show existence.
To have a suitable induction hypothesis we show the slightly stronger result 
that in the conformance case the redex of the last configuration in the 
minor pre-run is an $\ell$-annotation only if the last redex in the  major pre-run is 
also an $\ell$-annotation. We gloss over that in this sketch.

Consider any $\ell$ and $ins$, and go by induction on $gs$.  
The idea is to explore the minor pre-run $hs$ determined by $ins$.

In the base case, $gs$  is $[ \rconfig{\initState} ]$; we take 
$hs := [ \rconfig{\initState} ]$ and $\alpha := \{ (0,0) \}$. This forms a conformance.

In the induction case, $gs$ has the form $fs\cat g$
and by induction we have $hs,\alpha$ with $(fs,hs,\alpha)$ a conformance, fiat, or failure for $\ell,ins$.
If it is a fiat or failure then so is $(fs\cat g,hs,\alpha)$, by Lemma~\ref{lem:extendFail}, and we are done.
If $ins$ has been exhausted, i.e., $(fs,hs,\alpha)$ is a conformance
such that $\inputs(hs)=ins$ and $\last(hs)$ is receptive, then
$(fs\cat g,hs,\alpha)$ is a conformance and we are done.
It remains to consider the case that $(fs,hs,\alpha)$ is a conformance such that $\alpha$ covers $fs$
and either $\inputs(hs)\prefix ins$ or $\last(hs)$ is active.
We proceed by cases of $\last(fs)$.

\begin{itemize}
\item If $\last(fs)$ is receptive (in which case $\inputs(hs) \prefix ins$) and the transition to $g$ is for a visible input
on some $\ell'\lleq\ell$, then (because the program is  well-specified) $\redex(g)$ is $\codef{assume}^{\ell'}~\Agr x$ for some $x$.
We construct $hs'\suffixeq hs$ by successive steps,
maintaining $\inputs(hs') \prefixeq ins$, until
$\redex(\last(hs'))$ is an $\ell$-annotation, or $ins$ is exhausted and a receptive configuration is reached, or there is divergence without reaching either.  

  \begin{itemize}
  \item If $\inputs(hs')=ins$ and $\last(hs')$ is receptive,
        then $(fs\cat g, hs', \beta)$ is a conformance, where 
        $\beta = \alpha \union \{ (\len{fs}-1, j) \mid \len{hs}\leq j \prefixeq \len{hs'} \}$.
  \item If an $\ell$-annotation is reached at $\last(hs')$ then there are three sub-cases:

    \begin{itemize}

    \item If the annotation $\redex(\last(hs'))$ is the same as $\redex(g)$, and 
$g\mid \last(hs')\models\Agr x$ (i.e., $\state(g),\state(\last(hs'))$ agree on $x$) 
      then $(fs\cat g, hs', \gamma)$ is a conformance, where
      $\gamma = \alpha \union \{ (\len{fs}-1, j) \mid \len{hs} \leq j < \len{hs'}-1 \}
                       \union \{ (\len{fs}, \len{hs'}-1 ) \}$.
    \item If the annotation is the same as $\redex(g)$ but the states do not agree on $x$,
      then $(fs\cat g, hs', \delta)$ is an assumption fiat
or assertion failure, where 
      $\delta = \alpha \union \{ (\len{fs}-1, j) \mid \len{hs} \leq j < \len{hs'}-1 \}
      \union \{ (\len{fs}, \len{hs'}-1 ) \}$.
    \item If $\redex(\last(hs'))$ differs from $\redex(g)$ 
      then $(fs\cat g, hs', \delta)$ is an alignment failure, where $\delta$ is as in the preceding bullet.  
    \end{itemize}
  \item If neither $ins$ is exhausted nor an $\ell$-annotation reached while growing $hs'$,
    then the minor pre-run is diverging without 
    doing further visible input or output.  So $(fs\cat g, hs,\alpha)$ is a divergence fiat.
  \end{itemize}

\item If $\last(fs)$ is receptive and the transition to $g$ is for a non-visible input,
so $\redex(g)$ is an assumption for some $\ell'\nlleq\ell$,
then $(fs\cat g, hs, \alpha \union \{ (\len{fs},\len{hs}-1) \} )$ is a conformance.
Here we use that $\redex(\last(fs))$ cannot be an annotation, so neither can $\redex(\last(hs))$
(recall Remark~\ref{rem:one}).

\item If $\last(fs)$ is active, we have these sub-cases:
  \begin{itemize}
  \item If $\redex(g)$ is not an $\ell$-annotation then 
    $(fs\cat g, hs, \alpha\union \{ (\len{fs},\len{hs}-1) \})$ is a conformance
    that covers $fs\cat g$.
  \item If $\redex(g)$ is an $\ell$-annotation then construct $hs'\suffixeq hs$ by successive steps,
    maintaining $\inputs(hs') \prefixeq ins$, until $\redex(\last(hs'))$ is an $\ell$-annotation.
    \begin{itemize}
    \item If this is not possible because $\inputs(hs')=ins$ and $\last(hs')$ is receptive,
      then $(fs\cat g, hs', \beta)$ is a conformance, where $\beta$ is defined like earlier in this proof.
      (That is, $\beta = \alpha \union \{ (\len{fs}-1, j) \mid \len{hs}\leq j < \len{hs'} \}$.)
    \item If an $\ell$-annotation is reached, but $\redex(\last(hs'))$ is different from $\redex(g)$,
      then we obtain an alignment failure or alignment fiat, according to whether both $\redex(g)$ and $\redex(\last(hs'))$ are boilerplate assumptions.
    \item If $\redex(\last(hs'))$ is the same as $\redex(g)$, 
and the annotation's formula is $\Phi$
such that $g,\last(hs')\models\Phi$, 
      then $(fs\cat g, hs',\gamma)$ is a conformance, where
      $\gamma = \beta \union  \{ (\len{fs}, \len{hs'}-1) \}$
      for $\beta$ from above. 
    \item If $\redex(\last(hs'))$ is the same as $\redex(g)$ but the formula does not hold, 
      we obtain an assumption fiat or assertion failure.
    \item If none of the above apply, then we obtain divergence fiat
      at $i,k$ where $k=\len{fs}$ and $i$ is the index of the last $\ell$-annotation before $k$ in $fs$, truncating the minor pre-run where it matches $i$ and restricting $\alpha$ accordingly. 
    \end{itemize}
  \end{itemize}
\end{itemize}
\end{proof}

\thmsafesecure*
\begin{proof}
Consider any $\ell$.
In accord with the definitions, we consider an arbitrary pre-run and go by induction on it.

The base case is the shortest pre-run, $[\rconfig{\initState}]$, which is secure
by definition: it has no prefixes $hs\cat g$ with $hs$ nonempty.

For the induction step, consider a pre-run $gs\cat g$ that is safe.
Safe pre-runs are prefix closed (Lemma~\ref{lem:contractSafe}),
so $gs$ is safe, hence $gs$ is secure by induction hypothesis.
So to prove $gs\cat g$ is secure it remains to consider the last step.
We must show
\begin{equation}\label{eq:seccond}
\know^\ell(ts\cat t) \supseteq \pknow^\ell(ts) \intersect \RP^\ell(gs\cat g)
\end{equation}
for the unique (Lemma~\ref{lem:traceDetEv}) $ts,t$ that satisfy 
\begin{equation}\label{eq:gs}
gs \events ts \qquad
gs\cat g\events ts\cat t 
\end{equation}
By Lemma \ref{lem:knowlSame}(a), we have
$\know^\ell(ts\cat t) =  \pknow^\ell(ts)$ unless $t$ is $\ell$-visible,
and this proves (\ref{eq:seccond}) for all transitions except for visible $t\in \Ve^\ell$.

To prove (\ref{eq:seccond}) for $t\in \Ve^\ell$,
consider any list of inputs $ins$.
By safety of $gs\cat g$ for $\ell,ins$ we have $\fs ,\alpha$
such that 
$(gs\cat g, \fs , \alpha)$ is a conformance, assumption fiat, or divergence fiat for $\ell,ins$.
We must prove that 
\[ ins\in \pknow^\ell(ts) \mbox{ and } ins\in \RP^\ell(gs\cat g) \mbox{ imply } ins \in \know^\ell(ts\cat t) \]
either by refuting one of the antecedents or by showing the consequent.

By definitions, $ins\in \pknow^\ell(ts)$ means there are
$hs,us,ws$ and $u\in\Ve^\ell$ with 
\begin{equation}\label{eq:progress}
ins = \inputs(us\cat u\cat ws) \qquad
hs\events us \cat u \cat ws \qquad
\vis^\ell(ts)\prefixeq \vis^\ell(us)
\end{equation}
Also, the goal $ins \in \know^\ell(ts\cat t)$ means there are $\fs',vs$ such that
\begin{equation}\label{eq:show}
ins = \inputs(vs) \qquad
\fs'\events vs \qquad 
\vis^\ell(ts\cat t) \prefixeq \vis^\ell(vs) 
\end{equation} 
There are  quite a number of variables in play so let us review the situation.
The major pre-run $gs\cat g$ has trace $ts\cat t$.  
The minor pre-run $hs$, with trace $us\cat u$, witnesses that inputs $ins$ were possible according to prior knowledge, because $\vis^\ell(ts) \prefixeq \vis^\ell(us)$.  
Safety of $gs\cat g$ accounts for the inputs $ins$ by a pre-run $\fs$ that is either in conformance or is ruled out by assumption or divergence fiat.  

We complete the proof by cases on whether 
$(gs\cat g, \fs , \alpha)$ is a conformance, assumption fiat, divergence fiat,
or alignment fiat.
Conformance lets us derive $\fs'$ such that (\ref{eq:show}) holds.
Assumption fiat lets us refute the policy hypothesis  $ins\in \RP^\ell(gs\cat g)$, 
as does alignment fiat. 
Divergence fiat refutes the progress hypothesis $ins\in \pknow^\ell(ts)$.
Here are the details.

\textbf{Case conformance:}
Suppose $(gs\cat g, \fs , \alpha)$ is a conformance for $\ell,ins$.
So $\alpha$ covers $\fs$ and $\inputs(fs)\prefixeq ins$.
We go by the two cases of conformance: either $\alpha$ covers $gs\cat g$ or
$ins=\inputs(\fs)$ and $\last(\fs)$ is receptive.

Suppose $\alpha$ covers $gs\cat g$.  
By coverage and conformance, $\vis^\ell(\fs)=\vis^\ell(ts\cat t)$.
(In particular, if $t$ is an input then $\redex(g)$ is an assumption
about the input value and if $t$ is an output then a preceding assertion ensures 
agreement on the output value; either way, $\last(\vis^\ell(\fs))$ is $t$.)
By $\inputs(\fs)\prefixeq ins$ (by conformance)
and $ins=\inputs(hs)$ (by (\ref{eq:progress})), using determinacy we can extend $\fs$ to
$\fs'\suffixeq \fs$ with $ins=\inputs(vs)$ where $vs$ are the events of $\fs'$, i.e., $\fs'\events vs$.
So $\vis^\ell(ts\cat t)\prefixeq\vis^\ell(vs)$ and (\ref{eq:show}) holds for $\fs',vs$.

Suppose $\alpha$ does not cover $gs\cat g$ but $ins=\inputs(\fs)$ and $\last(\fs)$ is receptive.
Let $vs$ satisfy $\fs\events vs$. To show (\ref{eq:show}) instantiated  with $\fs',vs:=\fs,vs$, it remains to show $\vis^\ell(ts\cat t)\prefixeq\vis^\ell(vs)$. 
Since $\last(\fs)$ is receptive, $\fs$ is the maximal run on $ins$, so from (\ref{eq:progress})
using determinacy we have $us\cat u\cat ws\prefixeq \fs$.
By conformance the visible annotations of $\fs$ align with those of $gs$ since 
we consider the case where $g$ isn't covered, so $u$ must be $t$ and we get
$\vis^\ell(ts\cat t)\prefixeq\vis^\ell(vs)$. 

\textbf{Case divergence fiat:} 
Suppose $(gs\cat g, \fs , \alpha)$ is a divergence fiat for $\ell,ins$.
Suppose the divergence fiat is at $i,k$ so that
$((gs\cat g)\take i, \fs, \alpha)$ is an $\ell$-conformance for $ins$ 
but $\redex((gs\cat g)_k)$ is an $\ell$-annotation and neither $\fs$ nor any of its continuations
reach an $\ell$-annotation or $\ell$-output as redex.
Thus pre-run $\fs$ and its continuations are not doing further visible events following $\last(\inputs(\fs ))$.
Owing to $\inputs(\fs )\prefixeq ins$ and determinacy, 
if the progress condition (\ref{eq:progress}) holds  
then the handler for $\last(\inputs(\fs ))$ progresses at least as far as event $u$.
So divergence fiat contradicts the progress hypothesis $ins\in \pknow^\ell(ts)$.

\textbf{Case assumption fiat:}
Suppose $(gs\cat g, \fs , \alpha)$ is an assumption fiat for $\ell,ins$.
Suppose the assumption fiat is at $i,j$ 
where $j=\len{\fs}-1$.
Thus 
$((gs\cat g)\take i,\fs \take j,\alpha)$ is a conformance
and $\redex((gs\cat g)_i)$ is an $\ell$-assumption of formula $\Phi$
(and so is $\redex(\fs_j)$) but 
$(gs\cat g)_i | \fs_j \not\models \Phi$.
Let $ins'$ be $\inputs(\fs \take j)$ 
hence $ins' = \inputs(\fs \take (j+1)) = \inputs(\fs)$,
because that step is not an input.
Note $ins'\prefixeq ins$.
We now show that the policy hypothesis $ins \in \RP^\ell(gs\cat g)$ yields a contradiction.
Suppose $hs'$ witnesses $ins \in \RP^\ell(gs\cat g)$.
So $ins = \inputs(us')$ for the $us'$ with $hs' \events us'$,
and $\SC^\ell((gs\cat g),hs',\beta)$ for some $\beta$.  
But $\inputs(ins)$ determines $hs'$ 
so we have $\fs\prefixeq hs'$ and the condition 
$(gs\cat g)_i \, | \, \fs_j \not\models \Phi$
contradicts $\SC^\ell((gs\cat g),hs',\beta)$.

\textbf{Case alignment fiat:}
Suppose $(gs\cat g, \fs , \alpha)$ is an alignment fiat for $\ell,ins$,
at $i,j$ where $j=\len{\fs}-1$.
Thus $((gs\cat g)\take i,\fs \take j,\alpha)$ is a conformance;
moreover $\redex((gs\cat g)_i)$ and $\redex(\fs_j)$ are distinct 
boilerplate $\ell$-assumptions.  
We now show that the policy hypothesis $ins \in \RP^\ell(gs\cat g)$ yields a contradiction.
Suppose $hs'$ witnesses $ins \in \RP^\ell(gs\cat g)$.
So $ins = \inputs(us')$ for the $us'$ with $hs' \events us'$,
and $\SC^\ell((gs\cat g),hs',\beta)$ for some $\beta$.  
Now $\inputs(\fs)\prefixeq ins$ because the alignment fiat is for $ins$.
Also $ins$ determines $hs'$ so we have $\fs\prefixeq hs'$.
By definition of $\SC^\ell$, in particular semi-proper $\ell$-alignment,
we must have $i\beta j$, owing to alignment of the preceding $\ell$-assumptions in $gs\cat g$ and $\fs$.
But semi-proper requires 
$\redex((gs\cat g)_i)$ to be the same as $\redex(\fs_j)$,
a contradiction.
\end{proof}

\fi

\end{document}